\def\year{2021}\relax
\documentclass[letterpaper]{article} 
\usepackage{aaai21}  
\usepackage{times}  
\usepackage{helvet} 
\usepackage{courier}  
\usepackage[hyphens]{url}  
\usepackage{graphicx} 
\usepackage{natbib}  
\usepackage{caption} 
\usepackage{algorithm}
\usepackage{algorithmic}
\usepackage{xcolor}
\usepackage{bbm}
\usepackage{subfig}
\usepackage{amssymb}
\usepackage{mathtools}
\usepackage{amsthm}
\usepackage[switch]{lineno}
\newtheorem{theorem}{Theorem}
\newtheorem{lemma}[theorem]{Lemma}

\title{Learning by Sampling and Compressing: Efficient Graph Representation Learning with Extremely Limited Annotations}

\author{
Xiaoming Liu\textsuperscript{\rm 1,*}~
Qirui Li\textsuperscript{\rm 1,*}~
Chao Shen\textsuperscript{\rm 1}~
Xi Peng\textsuperscript{\rm 2}~
Yadong Zhou\textsuperscript{\rm 1}~
Xiaohong Guan \textsuperscript{\rm 1}
}

\affiliations{

    \textsuperscript{\rm 1} Ministry of Education Key Lab for Intelligent Networks and Network Security,\\ Xi'an Jiaotong University\\
    \textsuperscript{\rm 2} Deep Robust and Explainable AI Lab, University of Delaware.\\
     \textsuperscript{\rm *} Both authors contributed equally to this work.\\
    \{xm.liu,chaoshen,ydzhou,xhguan\}@xjtu.edu.cn, lqr1996@stu.xjtu.edu.cn, xipeng@udel.edu    
}

\begin{document}

\maketitle

\begin{abstract}
Graph convolution network (GCN) attracts intensive research interest with broad applications.
While existing work mainly focused on designing novel GCN architectures for better performance, few of them studied a practical yet challenging problem: How to learn GCNs from data with extremely limited annotation?
In this paper, we propose a new learning method by sampling strategy and model compression to overcome this challenge. 
Our approach has multifold advantages:
1) the adaptive sampling strategy largely suppresses the GCN training deviation over  uniform sampling;
2) compressed GCN-based methods with a smaller scale of parameters need fewer labeled data to train;
3) the smaller scale of training data is beneficial to reduce the human resource cost to label them.
We choose six popular GCN baselines and conduct extensive experiments on three real-world datasets. 
The results show that by applying our method, all GCN baselines cut down the annotation requirement by as much as 90$\%$ and compress the scale of parameters more than 6$\times$  without sacrificing their strong performance.
It verifies that the training method could extend the existing semi-supervised GCN-based methods to the scenarios with the extremely small scale of labeled data.
\end{abstract}

\section{Introduction} \label{sec:1}
Graph is a natural way to represent and organize data with complicated relationships. 
But graph data is hard to process by the machine learning methods directly, especially the deep learning \cite{lecun2015deep}, which has achieved brilliant achievements in various fields. Learning a useful graph representation lies at the heart of many deep learning-based graph mining applications, such as node classification, link prediction, and community detection \cite{zhang2020seal}, etc. 
It is now widely adopted to embed the structure data into vectors for well-developed deep learning methods. 

Recently, a semi-supervised method represented by graph convolutional network(GCN) has been a hot topic in the graph embedding area, and massive outstanding works are proposed\footnote{We will use convolutional graph network as the representative of semi-supervised graph embedding methods through the paper}. 
Kipf\shortcite{kipf2017semi} came up with GCN, the one widely used today, which has formally brought the field of graphs into the neural networks' era. 
Since then, plenty of work like GraphSAGE\cite{hamilton2017inductive} and graph attention networks(GAT)\cite{velivckovic2017graph} are proposed to achieve better performance.

\begin{figure}[!t]
  \centering
  \subfloat[Uniform Sampling\label{subfig1}]{%
  \includegraphics[width=0.23\textwidth]{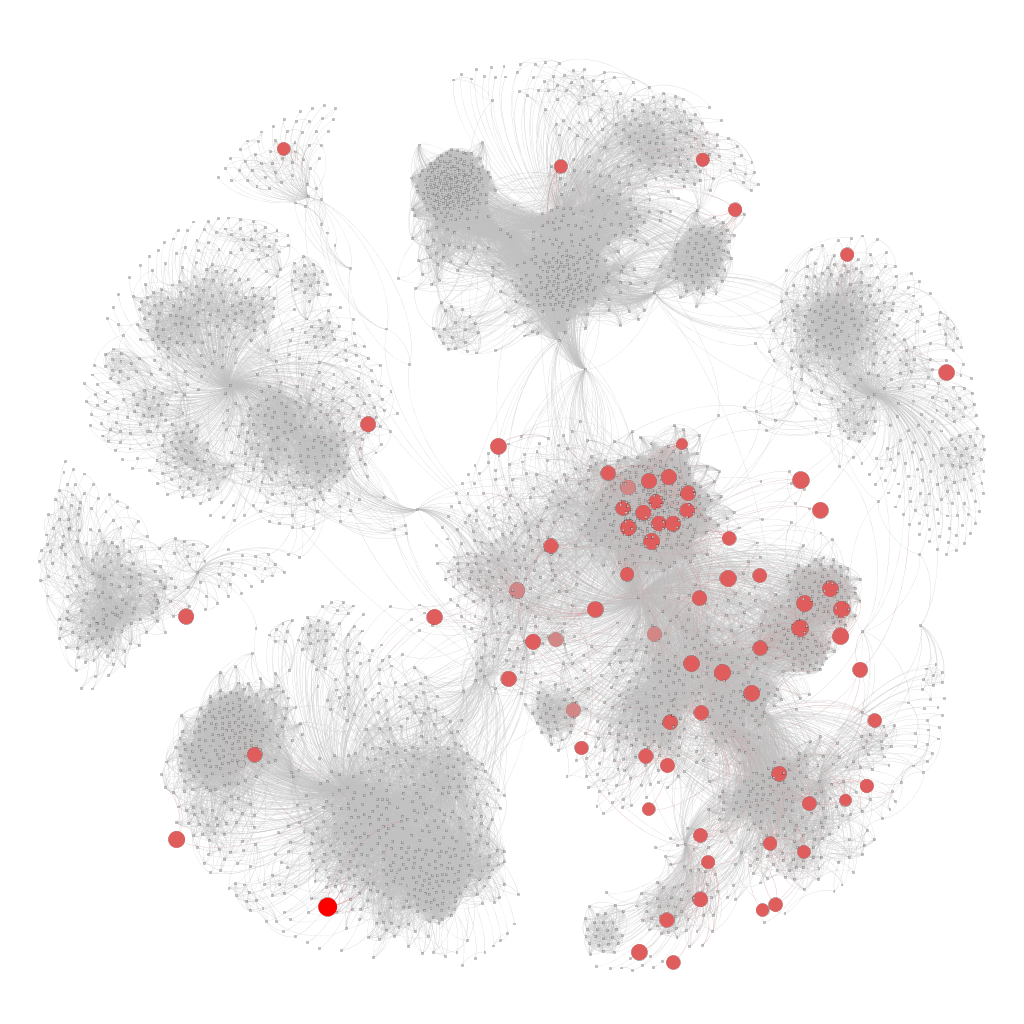}
  }
  \hfill
  \subfloat[Random Walks Sampling\label{subfig2}]{%
  \includegraphics[width=0.23\textwidth]{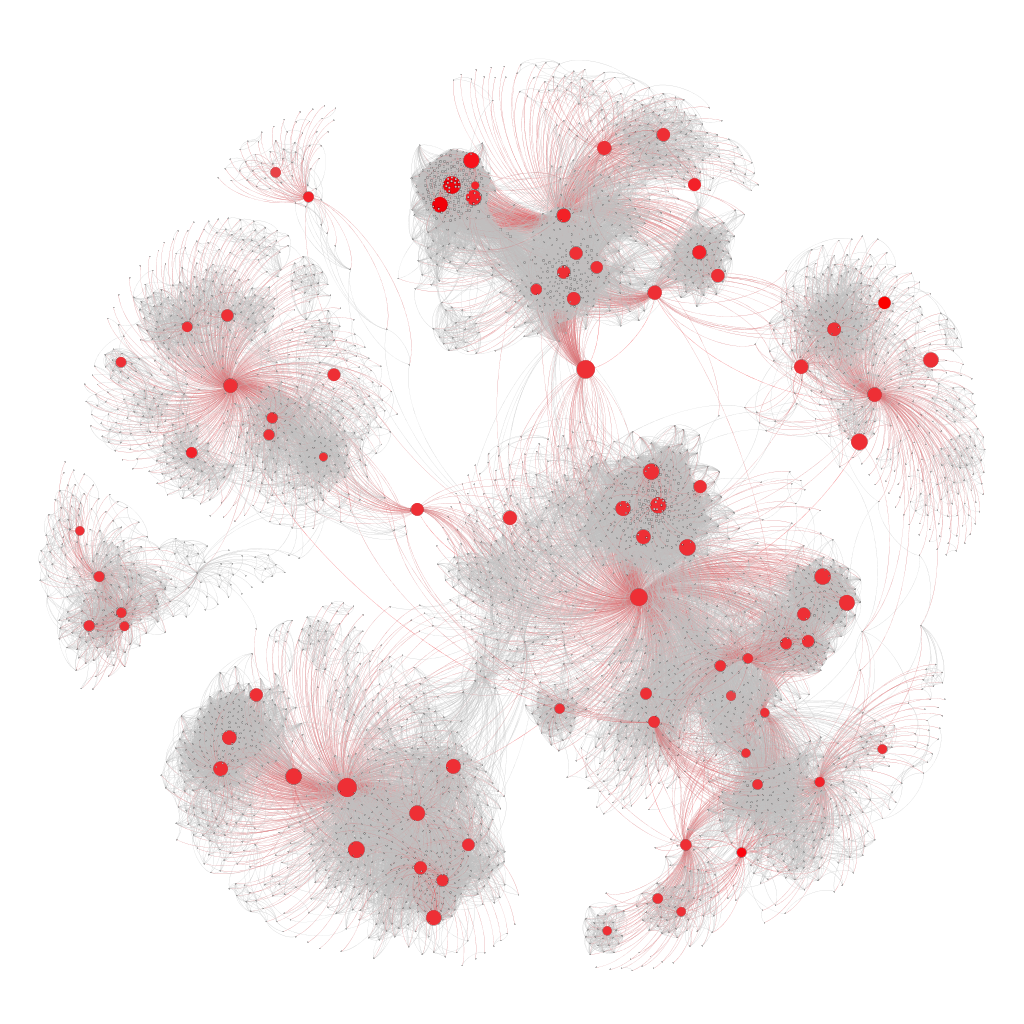}
  }
  \caption{Comparison between uniformly sampled training nodes and random walks-based sampled training nodes in limited ratio (red nodes).}
  \label{img1}
   \vspace{-0.5cm}
 \end{figure}
 
However, there are two key challenges in applying these semi-supervised methods to specific fields: 
1) extremely limited annotation, which is far from enough to train the model well; 
2) out-of-distribution prediction, i.e., the training set's distribution differs from the test set much. 
Works like GCN\cite{kipf2017semi} and GAT\cite{velivckovic2017graph} tend to manually pick the training set to maintain the distribution-similarity. 
On the other hand, some researchers seek to overcome the challenges by utilizing transfer learning\cite{hu2019pretraining,yun2019graph}. 
But the `pre-train' methods need extensive domain knowledge and pretty long training time.

To address the research gap,  we proposed an efficient training framework for GCN-based methods, especially in extremely limited labeled data, which is more scalable and needless of domain knowledge.
We integrate a random walks-based sampling strategy with the model compression method into the GCN-based model training process. 
In this way, our framework utilizes the sampling algorithm to find out the most representative nodes of a graph based on the multiple dependent random walks\cite{liu2014detecting}, which have proven successful in graph measurement and graph structure estimating. 
The comparison between random walks-based sampled nodes and uniformly chosen nodes with a limited ratio is shown in Fig. \ref{img1}.
When the sampling scale is small,  nodes chosen uniformly are more likely to be in a small sub-graph, which would make the GCN model learning bias during the training process. 
Besides, we take advantage of the matrix compression method to reduce the scale of parameters in the GCN-based model, which is also beneficial to reduce the requirement of annotations.

 To demonstrate the proposed framework's validation in the training process, we implement it with six latest GCN-based methods on three real-world datasets.
 Their performances are evaluated on the challenging multi-label node classification problem with limited labeled training data. 
 The result shows that each implemented method can outperform its original one with the same scale of training data, or achieve the same accuracy with fewer labeled data.
 
The contribution of this paper is summarized as follows:
\begin{itemize}
    \item We utilize the sampling strategy based on multiple dependent random walks as a preprocessing stage to improve the performance of GCN-based methods. 
    It can reduce the requirement for the scale of labeled data and improve the performance without changing the original methods.  
    We also develop a matrix decomposition method to compress the GCN-based model, which reduce the scale of parameters from $O(b*c)$ to $O(d*r*\max(b, c))$, where $b$ and $c$ are the dimensions of the weight matrix $W$, and $d, r \ll \min(b,c)$.
    The compressed model would need fewer labeled data to train.
    \item
    We propose a general training framework for GCN-based methods by integrating a random walks-based sampling strategy with the model compression method, which could make original GCN-based methods obtaining better results and dealing with the application using an extremely small scale of labeled data.
    \item We implement the proposed training framework to six GCN-based methods on three different real-world datasets.
    The numeral evaluation on the multi-label node classification problem shows that our framework can make GCN-based methods achieving the original performance with just 10$\%$-50$\%$ of the scale of training data.
    Besides, the parameters' scale is compressed by more than 6$\times$ with 16$\%$ additional time cost. 
\end{itemize}

\section{Problem Definition} \label{sec:2}

 Most of the state-of-the-art GCN-based methods are faced with one common training problem, i.e., the requirement for a large number of labeled nodes as the training set to achieve satisfying performance.
We seek to solve this problem by proposing an efficient training framework, in which the sampling strategy is integrated with the model compression method.
In this way, we only need to label the sampled tiny-scale nodes set to train the GCN-based methods with reduced parameters, which can achieve or even surpass the original models' performance obtained by  much larger training data.

\begin{figure}
  \centering
  \includegraphics[width=0.48\textwidth]{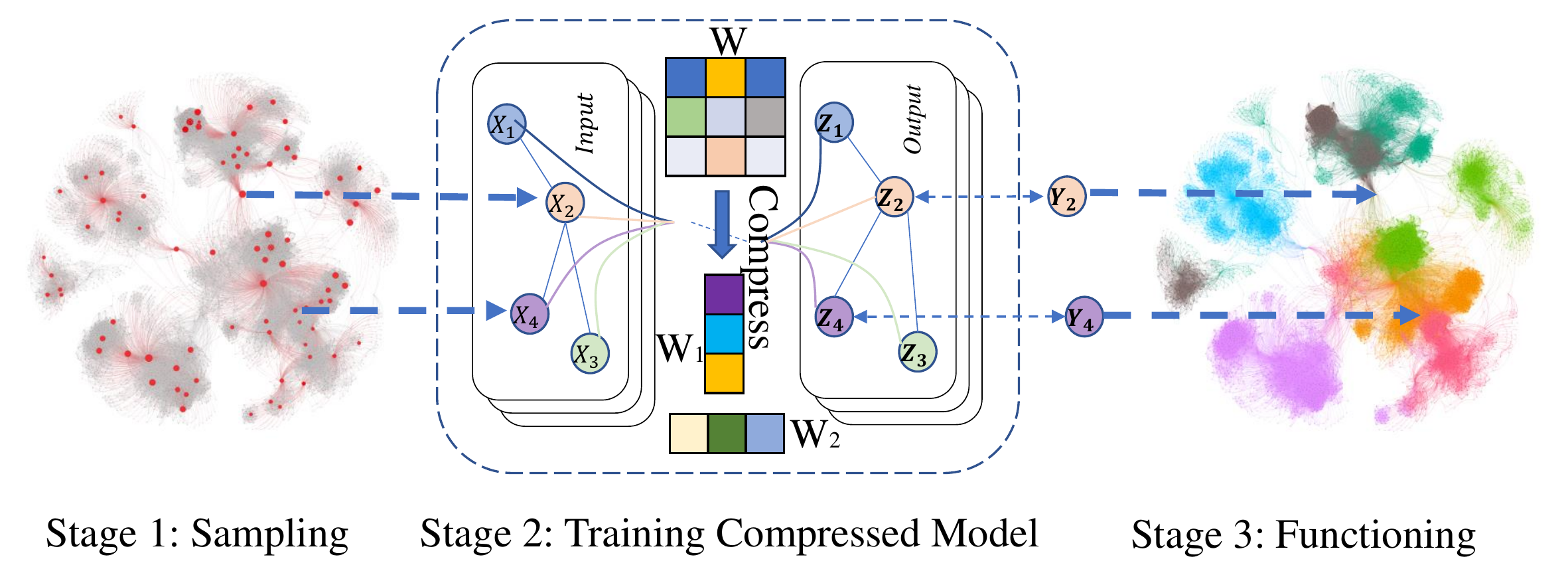}
  \caption{Framework of the proposed training framework.}
  \label{img2}
   \vspace{-0.5cm}
\end{figure}

This problem has attracted the attention of researchers, and the most related works are FastGCN\cite{chen2018fastgcn} and GraphSAGE\cite{hamilton2017inductive}. 
They also utilize sampling methods to improve the performance based on the aggregation feature selection, which is simplified as \textit{F}-sampling.
But it is necessary to clarify that their sampling strategies and ours mentioned in this paper are not the same concepts.
Especially, both FastGCN and GraphSAGE focus on the strategy to sample the neighbor's information during the training computation process.
We, differently, use a sampling strategy as a preprocessing method by selecting nodes to be labeled out of the training computation process, which is simplified as \textit{N}-sampling. 

In short, \textit{F}-sampling strategy aims to select \textbf{features} for the training nodes by changing the workflow of original models, but \textit{N}-sampling is determined to select the more representative training \textbf{nodes} without changing anything of the original model.
Besides, we also design a general model compression method by matrix decomposition theory, which can reduce the original GCN model's parameter scale.
It can also make the model more adaptable to extremely limited labeled data.
The workflow of our idea is shown in Fig. \ref{img2}.
In this way, our proposed framework is independent of the GCN-based methods and can be extended easily to almost every graph representation method.

\section{Methodology} \label{sec:4}

\subsection{Workflow of the Training Framework}
Our framework can be generalized into a three-stage algorithm, which is illustrated in Fig. \ref{img2}.
First, we implement a sampling strategy on the graph $\mathcal{G}=(V, E, W)$ to determine which nodes need to be labeled, where $v \in V$ is one node, $e \in E$ is one edge representing a connection between two nodes, and $w \in W$ is the edge weight indicating the closeness of two nodes.
The sampling process can be formulated as 
\begin{equation}
    V_s, X_s, Y_s = f(V,X,Y,B),
\end{equation}
where $f$ is the sampling strategy,  \textit{B} denotes the expected sampling budge, $V_s$, $X_s$, and $Y_s$ represent the sampled nodes, the corresponding feature matrix and labeled embedding vectors, respectively.

Then, model compression based on the matrix decomposition is applied to the GCN-based method, which could reduce the scale of parameters significantly.
The key computation equation in GCN-based model is transformed as 

\begin{align}
\notag 
F(H^{l}, A) = \sigma\left( \hat{D}^{-\frac{1}{2}}\hat{A}\hat{D}^{-\frac{1}{2}}H^{l}W^{l}\right) \xRightarrow{\text{compress}} \\
 F_c(H^{l}, A) = \sigma\left(\sum \left( \hat{D}^{-\frac{1}{2}}\hat{A}\hat{D}^{-\frac{1}{2}}H^{l}_d\prod{G_d[i_d,j_d]} \right) \right),
\end{align}
where $A$ and $D$ are the adjacent matrix and diagonal matrix, $H^{l}$ is the output of $l$-th neural network layer, and $G_d$ is a kernel matrix for the decomposition of the weigh matrix $W^{l}$, whose rank is much lower than $W^{l}$.  

Finally, the proposed training framework based on the sampling data for limited annotations is generalized as
\begin{equation}
  Model = F_c(V_s,X_s,Y_s)
\end{equation}
where $F_c$  represents the compressed model.
After the model has been well trained, it can be functional as usual.


Our framework satisfies the following characteristics:
\begin{itemize}
  \item Small scale - Our original intention is to downsize the training set for a graph convolutional network, so the graph embedding model could be trained well with limited labeled data;
 \item Unbiased - Sampling strategy aims to find out the well-distributed nodes which are more representative in the graph, and can depict the structure of the graph;
  \item Low time complexity - The time complexity for the sampling algorithm and model compression is nearly linear, which would not increase the time complexity of the original model;
  what's more, the well-chosen training dataset can make the model converging quickly, which has a good chance to lower the computational training cost; 
  \item  Scalability - The computation of the sampling stage and training stage are independent, and the matrix decomposition is a general method, which make our framework  easily implemented on different models and different datasets.
\end{itemize}

\subsection{Sampling Strategy}

The GCN-based models usually adopt the `uniformly random sampling'(UR) method to choose the training dataset, in which one node is sampled by the probability of $p(v)=\frac{1}{n}$, where $n$ refers to the number of nodes in one graph.
This method lacks considering the structure of the graph, which leads to unstable performance for its performance.

In order to overcome the disadvantage of UR, we take advantage of random walks(RW) sampling strategy to select the representative nodes as the training data.
We adopt the random walks as the sampling strategy for several considerations: 1) random walks is a wildly used method in graph structure estimation \cite{liu2016miracle}; 2) it can obtain the well-distributed nodes with linear computation time complexity.

The random walks-based sampling algorithm starts from a root node \begin{math}v_i\end{math}.
After pushing \begin{math}v_i\end{math} into the  traversed node list \begin{math}L\end{math}, it chooses the next sampling node \begin{math}v_j\end{math} from the neighbours $N(v_i)$ of \begin{math}v_i\end{math} based on the probability,
\begin{equation}
    p(v_j)=\frac{w_{ij}}{\sum \nolimits_{v_k \in N(i)}w_{ik}},
\end{equation}
where $w_{ij}$ is the weight/preference from $v_i$ to $v_j$, and $w_{ij}=1$ for unweighted graph.
When the number of sampled nodes reaches the budget $B$, the random walks stops, and the sampled training dataset is obtained.

But the sampling strategy based on single random walks is still faced with several technical challenges.
First, it has a relatively high demand for the structure of the graph.
For example, if the graph is unconnected, the single random walks would only explore the subgraph where it starts.
In this situation, the sampling training dataset inevitably leads to learning bias for their limited distribution.
Second, even if the graph is connected, the single walker is easily trapped in a closely connected subgraph.
This phenomenon would also affect the sampling performance heavily.

Inspired by the previous work\cite{ribeiro2010estimating}, we employ multiple random walks to overcome technical challenges.
It performs $m$ dependent walkers simultaneously, who share the candidate list $\mathcal{L}$ together.
$\mathcal{L}$ is initialized by uniformly choosing \begin{math}k\end{math} nodes from the global graph.
When a walker coming to its $t$-th hop \begin{math}v^t\end{math},
the walker's next hop \begin{math}v_j^{t+1}\end{math} is selected from the list \begin{math}\mathcal{L}^t\end{math}  with the probability $p(v_j^{t+1})$,
\begin{equation}\label{eq:4}
      p(v_j^{t+1}) = \frac{w(v_j^{t+1})}{\sum_{\forall{v_k\in\mathcal{L}^t}}w(v_k)}
\end{equation}
where $ w(v_k)=\sum \nolimits_{v_i \in N(k)} w_{ki}$.
Then, we replace $v_j$ by $v_u \in N(j)$ in $\mathcal{L}^t$, and obtain $\mathcal{L}^{t+1}$.
Notice that $v_u$ is selected based on the weight of $w_{ju}$.

Thus, the multiple walkers\footnote{`walker' and `dimension' have the same definition for the sampling process in this paper.} with the strategy of `Frontier Sampling'\cite{ribeiro2010estimating} are less likely to be trapped in a closely connected part of the whole graph. 
In order to demonstrate the superiority of this sampling strategy, we take the label density prediction, one important graph characteristic, to indicate the distribution of the sampling nodes.

Detailedly, each node $v$ is associated with a label $Q(v) \in \{q_1,q_2,...,q_k\}$.
The label density on graph $G$ is represented by $\theta=\{\theta_1,\theta_2,...,\theta_j\}$, and $\theta_i$ is defined as, 
\begin{equation} \label{equation:density}
    \begin{aligned}
     \theta_i = \frac{1}{n} \sum_{v \in V} \mathbbm{1}(Y(v)=q_i), 
    \end{aligned}
\end{equation}
where $\mathbbm{1}$ is the indicator function, which  assigned the value by 1 if $Y(v)=q_i$, otherwise 0.

Based on the sampling data, the unbiased estimator \cite{zhao2019sampling} for label density is defined as
\begin{equation} \label{estimator}
    \begin{aligned}
     & \hat{\theta_i} = \frac{1}{C} \sum_{v \in V_s} {\frac{\mathbbm{1}(Y(v)=q_i)}{\pi(v)}},\\
     & C = \sum_{v \in V_s} \frac{1}{\pi(v)},
    \end{aligned}
\end{equation}
where $\pi(v)$ is the probability that node $v$ is sampled.
Especially, $\pi(v)$ equals to $1/n$ in uniformly randomly sampling, and $w(v)/(\sum \nolimits_{v_j \in V}w(v_j))$ in random walks at steady state, respectively. 

\begin{theorem} \label{theorem1}
For a single random walker, $(\hat{\theta}_i^{RW}-\theta_i)^2 \leq (\hat{\theta}_i^{UR}-\theta_i)^2$.
\end{theorem}

\begin{proof}

Combining the equation \ref{equation:density} and equation \ref{estimator}, the original inequality can be written as
\begin{equation}\label{eq: 8}
    \begin{aligned}
    & (\hat{\theta}_i^{RW})^2 - \frac{2}{n}\sum_{v \in V}\mathbbm{1}(Y(v)=q_i) \hat{\theta}_i^{RW} 
    \leq \\
    & \frac{(\sum_{v \in V_s^{UR}} \mathbbm{1}(Y(v)=q_i))^2}{B^2} \ - \\ &\frac{2\sum_{v \in V}\mathbbm{1}(Y(v)=q_i) \sum_{v \in V_s^{UR}}\mathbbm{1}(Y(v)=q_i)}{nB},
    \end{aligned}
\end{equation}
where  $B$  represents the number of the sampled nodes. 
By simplifying equation \ref{eq: 8}, we can obtain 
\begin{equation}\label{eq: 9}
        nB\hat{\theta}_i^{RW} \leq 2B\sum_{v \in V}\mathbbm{1}(Y(v)=q_i).
\end{equation}

Because $\mathbb{E}(\hat{\theta}_i^{RW}) = \frac{1}{n}\sum_{v \in V}\mathbbm{1}(Y(v)=q_i)$, then we have $2\sum_{v \in V}\mathbbm{1}(Y(v)=q_i)-n\hat{\theta}_i^{RW} > 0$, which makes equation \ref{eq: 9} satisfied.
\end{proof}

 Theorem \ref{theorem1}  indicates that the label density $\hat{\theta}_i^{RW}$  estimated by node sequence $V_s^{RW}$ sampled by random walks is closer to $\theta_i$ than  $\hat{\theta}_i^{UR}$ estimated by uniformly randomly sampled sequence $V_s^{UR}$.
 This result reveals that the sample nodes by random walks have better distribution than the ones by uniformly randomly sampled sequences.
 
 

\begin{lemma} \label{lemma1}
The adopted m-dimensional dependent random walks process performs better on label density estimation than a single random walker.
\end{lemma}
Lemma \ref{lemma1} has been proved by the conclusion of previous work \cite{ribeiro2010estimating}.
Theorem \ref{theorem1} has shown better performance of the single random walks over uniform sampling, and the adopted sampling method in this paper out-performs the single random walks proved in lemma \ref{lemma1}. 
We can deduce theorem \ref{theorem2} by combining theorem \ref{theorem1} and lemma \ref{lemma1}.

\begin{theorem} \label{theorem2}
The nodes sampled by the multiple dependent random walks have better distribution than nodes sampled uniformly based on the label density estimation.
\end{theorem}

Therefore, the theoretical analysis indicates that the adopted sampling strategy based on random walks could obtain more representative nodes. 
It is significant to train the GCN model with limited data efficiently.

\subsection{Model Compression}\label{sec:com}
Except for the sampling strategy to reduce the training dataset scale, we also propose a matrix decomposition-based method to compress the GCN models.
Detailedly, we utilized the Tensor-Train compression method \cite{oseledets2011tensor} to reduce the size of a GCN-based method, which could lower the scale of parameters to be estimated.


The targeted convolution layer in GCN can be simplified as 
\begin{equation}\label{gcn-origin}
    \begin{aligned}
        H^{l+1} = A*H^{l}*W
    \end{aligned}
\end{equation}
where $W$ is the weight matrix of parameters to be estimated, $A$ is the adjacent matrix and $H^{l}$ is the output of the $l$-th neural network layer.

The tensor-train format of can be easily applied on equation \ref{gcn-origin}, which transforms it into 
\begin{equation}\label{gcn-decompose}
    \begin{aligned}
        \mathcal{H}^{l+1}(i_1,...,i_d) = \sum{A\mathcal{H}^{l}(j_1,...,j_d)G_1[i_1,j_1]...G_d[i_d,j_d]}
    \end{aligned}
\end{equation}
where $W$ with rank 2 is decomposed into the format as the product of a series of kernel matrices $G$.

After the computation transformation as equation \ref{gcn-decompose}, the parameters needed to be estimated are the kernels $G_k$, whose largest rank is denoted as tt-rank $r$ \cite{grasedyck2011introduction}. 
The number of model parameters will be shrunk from $b*c$ to no more than $d*r*\max(b, c))$, where $b, c$ are the scale of two dimensions for the weight matrix $W$, $d$ is the number of kernels. $r$ and $k$ are usually much smaller than $b$ and $c$,
which makes this kind of format capable of reducing the scale of parameters.
We denote the equation \ref{gcn-decompose} as equation \ref{gcn-decompose-2} to highlight the $G_i$. 
\begin{equation}\label{gcn-decompose-2}
    \begin{aligned}
        \mathcal{H}^{l+1}(i_1,...,i_d) = 
        \sum{A\mathcal{H}^{l}(j_1,...,j_d)P_{k}^{-}G_{k}P_{k}^{+}}
    \end{aligned}
\end{equation}

$P_{k}^{-}$ represents the production of $G_t(t<k)$, and $P_{k}^{+}$ is the production of $G_t(t>k)$. The training process for $G_k$ would be computed as equation \ref{gcn-chain-rule} and \ref{gcn-chain-rule-2} using chain rule, where $L$ represents the Loss function.
\begin{equation}\label{gcn-chain-rule}
    \begin{aligned}
        \frac{\partial L}{\partial G_k} = \frac{\partial L}{\partial \mathcal{H}^{l+1}}\frac{\partial {H}^{l+1}}{\partial G_k}
    \end{aligned}
\end{equation}

\begin{equation}\label{gcn-chain-rule-2}
    \begin{aligned}
        \frac{\partial {H}^{l+1}}{\partial G_k} = \sum{A\mathcal{H}^{l}(j_1,...,j_d)(P_{k}^{-})^{T} (P_{k}^{+})^{T}}
    \end{aligned}
\end{equation}

Equation \ref{gcn-chain-rule-2} can be calculated by dynamic programming \cite{novikov2015tensorizing}, and equation \ref{gcn-chain-rule} can be computed by substituting the results of equation \ref{gcn-chain-rule-2}.
We will carry on detailed experiments to evaluate the compression capacity.

\subsection{Implementation}

The proposed training framework provides an easy and effective way to decrease the training set scale. 
We intend to apply this framework on  GCN-based methods in a general way. 
We present the implementation details of the framework in Algorithm\ref{alg1}, which consists of three-stage.


\begin{algorithm}[tb]

\caption{Efficient Graph Representation Learning Framework}
\label{alg1}
\textbf{Input:} graph $\mathcal{G}$, test nodes $V_{test}$, corresponding test features $X_{test}$, GCN-based method $F$\\
\textbf{Parameter:} sampling budget $B$, number of the walkers $m$, compressed kernels rank $r$, threshold for early stopping $\epsilon$\\
\textbf{Output:} task-oriented $Predict$ 
\begin{algorithmic}[1]
\STATE Initialize $\mathcal{L}=(v_1,...,v_m)$ with $m$ uniformly random chosen nodes
\STATE $i\leftarrow0$, $V_s \leftarrow \emptyset$
\WHILE {$i < B$}
  \STATE Select $v\in \mathcal{L}$ and  $v\notin V_s $ with probability calculated by equation \label{eq:4}\\
   \STATE Add $v$ into list $V_s$
  \STATE Select a neighbour $u \notin V_s $ of $v$ based on the edge weight $w(v,u)$
  \STATE Replace $v$  by $u$ in $\mathcal{L}$
  \STATE $i \leftarrow i + 1$
\ENDWHILE   \textit{~$\#$Sampling Process End}
\STATE $W=G_{1}G_{2}...G_{d}$
\STATE Substitute $W$ in $F(H^{(l), A})$ with result of line 10 and obtain the compressed model $F_c(H^{(l)}, A)$
\STATE $Model = F_c(V_s,X_s,Y_s)$ 
\WHILE{condition hold} 
  \STATE Calculating $\frac{\partial L}{\partial G_k}$ with by equation \ref{gcn-chain-rule-2}, using sampled nodes $Y_s$
  \STATE Updating $G_k$ and calculating error $\delta_i$
\ENDWHILE
\textit{~$\#$ Training the compressed model by the sampling data}
\STATE $Predict = Model(V_{test},X_{test}),  V \notin S$ \newline   \textit{~$\#$ Task oriented functioning process} 
\STATE \textbf{return} $Predict$
\end{algorithmic}

\end{algorithm}

\subsection{Time Complexity}
According to Algorithm \ref{alg1}, the proposed framework's time complexity is determined by the sampling method and model compression method.
For the sampling strategy based on random walks, its time complexity $T_1$  depends only on the scale of the sampling budget $B$
\begin{equation}
  T_1 \propto B,\\ B = \frac{n}{k},\\ k \in (1,+\infty),
\end{equation}
where $n$ represents the number of nodes in the graph $\mathcal{G}$. 
Thus, the time complexity of the sampling algorithm is lower than $O(n)$. 
For the model compression, its time complexity $T_2$ is related to the matrix decomposition method, which is linear to the number of nodes $n$,
\begin{equation}\label{time:2}
    T_2=O(r^4n),
\end{equation}
where $r$ is the hyper parameter for decomposition, and $r \ll n$.
Thus, the proposed training framework's total time complexity is $T_1+T_2 \approx O(n)$.

\section{Evaluation}\label{sec:5}
We verify our proposal by the multi-class classification task on three real-world datasets, including two citation networks and one social network. 
In the citation networks, the nodes are papers and the edges are the citation relationship. 
Each paper has a feature vector that contains the information of its contents. 
Classes implicate the kind of categories among the papers. 
And for the social network, the nodes represent users using social media, and an edge between two users means the follower-followed relation. 
And the details for these datasets are presented in Table \ref{table1}.

\begin{table}
  \centering
  \caption{Overview of Graph Datasets}
  \begin{tabular}{lcrrrc} \hline
  Dataset&Type&Node&Edges&Classes\\ \hline
  Cora & Citation& 2,707 &5,429&7\\ 
  Pubmed & Citation &  19,717 &44,338&3\\
  BlogCatalog & Social & 10,312 & 333,983 & 10 \\
  \hline\end{tabular}
  \label{table1}
\end{table}

\subsection{Experimental Settings}
We implement our framework on six popular GCN-based methods to verify its validity. 
For the sampling strategy, we set the number of multiple random walks as $m=3$.
For the training data scales(sampling budget), we range the sampling ratio in $[0.5\%,1\%,5\%,10\%]$ for each dataset.
For model compression, we set the decompose tt-rank $r$ as 8.
For the baseline algorithms, we choose the same scale of training data from the graphs by uniformly sampling strategy; 
One hundred nodes are randomly selected from the training set as the validation part.
The prediction accuracy is evaluated on another randomly selected 1000 nodes for each dataset. We use the `Cross-Entropy Loss'\cite{zhang2018generalized} as our loss function during the experiments.
We use `SS-' with the original method name to represent methods combined with our framework.
We perform each experiment 10 times and take the average results as the final results.
The experiment on FastGCN is based on the code released by the original authors\footnote{https://github.com/matenure/FastGCN}, and all the other algorithms are implemented based on the Deep Graph Library (DGL)\footnote{https://github.com/dmlc/dgl}.


\begin{table*}[!t]
\centering
\caption{Accuracy on Multi-label Classification}  
\begin{tabular}{lllll|llll|llll}
\hline
             & \multicolumn{4}{c}{Cora}                                                           & \multicolumn{4}{c|}{Pubmed}                                   & \multicolumn{4}{c}{BlogCatalog}                               \\ \cline{2-13} 
Training Set Ratio & 0.5\%         & 1\%           & 5\%           & \multicolumn{1}{l|}{10\%}          & 0.5\%         & 1\%           & 5\%           & 10\%          & 0.5\%         & 1\%           & 5\%           & 10\%          \\ \hline
GCN          & 0.63          & 0.64          & 0.73          & \multicolumn{1}{l|}{0.80}          & 0.39          & 0.65          & 0.78          & 0.79          & 0.25          & 0.30          & 0.31          & 0.33          \\
SS-GCN       & \textbf{0.70} & \textbf{0.71} & \textbf{0.75} & \multicolumn{1}{l|}{\textbf{0.83}} & \textbf{0.63} & \textbf{0.73} & \textbf{0.78} & \textbf{0.81} & \textbf{0.28} & \textbf{0.30} & \textbf{0.33} & \textbf{0.34} \\ \hline
GraphSAGE    & 0.63          & 0.57          & 0.79          & \multicolumn{1}{l|}{0.85}          & 0.70          & 0.79          & 0.82          & 0.83          & 0.27          & 0.27          & 0.32          & 0.34          \\
SS-GraphSAGE & \textbf{0.69} & \textbf{0.66} & \textbf{0.86} & \multicolumn{1}{l|}{\textbf{0.85}} & \textbf{0.78}          & \textbf{0.85} & \textbf{0.85} & \textbf{0.85} & \textbf{0.33} & \textbf{0.32} & \textbf{0.34} & \textbf{0.36}          \\ \hline
SGC          & 0.47          & 0.61 & 0.79          & \multicolumn{1}{l|}{0.83}          & 0.78          & 0.79          & 0.81          & 0.83          & 0.25          & 0.27          & 0.32          & 0.33          \\
SS-SGC       & \textbf{0.56} & \textbf{0.68} & \textbf{0.81} & \multicolumn{1}{l|}{\textbf{0.84}} & \textbf{0.81} & \textbf{0.82} & \textbf{0.83} & \textbf{0.84}          & \textbf{0.33} & \textbf{0.33} & \textbf{0.33} & \textbf{0.35}          \\ \hline
FastGCN          & 0.18          & 0.26 & 0.33          & \multicolumn{1}{l|}{0.33}          & 0.56         & 0.59          & 0.63          & 0.63          & 0.19          & 0.23          & 0.27          & 0.27         \\
SS-FastGCN       & \textbf{0.23} & \textbf{0.32} & \textbf{0.38} & \multicolumn{1}{l|}{\textbf{0.38}} & \textbf{0.59} & \textbf{0.64} & \textbf{0.66} & \textbf{0.66}          & \textbf{0.25} & \textbf{0.27} & \textbf{0.30} & \textbf{0.31}          \\ \hline
TAGCN        & 0.56          & 0.56          & 0.79          & \multicolumn{1}{l|}{0.79}          & 0.70          & 0.77          & 0.80          & 0.83          & 0.30 & 0.30 & 0.33          & 0.32          \\
SS-TAGCN     & \textbf{0.70} & \textbf{0.71} & \textbf{0.80} & \multicolumn{1}{l|}{\textbf{0.84}} & \textbf{0.79} & \textbf{0.80} & \textbf{0.86} & \textbf{0.85} & \textbf{0.31}          & \textbf{0.32}          & \textbf{0.34} & \textbf{0.34} \\ \hline
APPNP        & 0.68          & 0.71          & 0.72          & \multicolumn{1}{l|}{0.80}          & 0.74          & 0.80          & 0.83          & 0.83          & 0.26          & 0.32          & 0.32          & 0.32          \\
SS-APPNP     & \textbf{0.72} & \textbf{0.79} & \textbf{0.84} & \multicolumn{1}{l|}{\textbf{0.85}} & \textbf{0.79} & \textbf{0.81} & \textbf{0.85} & \textbf{0.86} & \textbf{0.32} & \textbf{0.33} & \textbf{0.33} & \textbf{0.34} \\ \hline
\end{tabular}
\label{table2}
\end{table*}

\subsection{Baseline Methods}
The state-of-the-art GCN-based baselines are summarized as follows:
\begin{itemize}
  \item GCN~\cite{kipf2017semi}: This is the first wildly used graph convolutional network method to embed the graph structure into vectors. 
  It takes the graph structure and a few labeled nodes as input and outputs the node embedding vector.
   \item GraphSAGE~\cite{hamilton2017inductive}: This method is built on GCN, which aggregates the neighbors' features to depict one node. 
   In this way, the framework can deal with dynamic graph structure.
  \item SGC~\cite{wu2019simplifying}: This method speeds up the GCN's training time by removing nonlinearities and collapsing weight matrices between consecutive layers.
  \item FastGCN~\cite{chen2018fastgcn}: This method also tries to speed up the GCN's training time by sampling active nodes between layers, which performs like 'dropout' in traditional neural networks.
  \item TAGCN~\cite{du2017topology}: This method designs a set of fixed-size learnable node filters to perform convolutions on graphs.
  It differs from the spectral domain of the origin GCN.
  \item APPNP~\cite{klicpera2018predict}: This method uses the idea of PageRank\cite{page1999pagerank} to improve the performance of GCN by utilizing propagation procedure to construct a simple model.
\end{itemize}

\subsection{Result Comparison}
We now validate the effectiveness of our framework by combining it with 6 GCN-related baseline algorithms and compare them with the original ones. 
Specifically, we use the task of node classification for evaluation. 
The experimental results are shown in Table \ref{table2}. 
We bold the better result for each comparison pair, and the detailed analysis is presented as follows.

When the training dataset ratio is set at $0.5\%$, we find that the ones with our framework can significantly outperform the original ones by $4\%-60\%$. 
The proposed training framework can obtain an improvement of $10\%$ on average with the extremely limited labeled data. 
And the greatest improvement happens with GCN and SS-GCN on the Pubmed dataset.

With the scale of the training dataset rising, the improvement is becoming smaller, but still exists.
We can get an improvement of $4\%$ on average when the training scale ratio is set as $10\%$.
Observing from another perspective on the prediction accuracy, we can learn about the significance of the proposed framework more clearly.
For example, taking the accuracy of SS-SGC on the Pubmed with $0.5\%$ dataset -- 0.81 as a goal, we notice that the original SGC needs up to $5\%$ of the training data to achieve this accuracy, which is 10 times of SS-SGC.
In other words, the proposed training framework can reduce $90\%$ need for labeled data to get the same accuracy.

Overall, we can summarize the conclusions drawn from the results:
1) Algorithms with the proposed training framework can get a $4\%-60\%$ accuracy improvement under the same situation.
2) Methods with the proposed training framework can get a close performance just with $10\%-50\%$ of the original training data scale. 
The outstanding results indicate that our framework can easily improve the GCN-related method performance, especially with extremely limited labeled data.


\subsection{Sampling Strategy Comparison}
\begin{figure}[!t]
  \centering
  \includegraphics[width=0.25\textwidth]{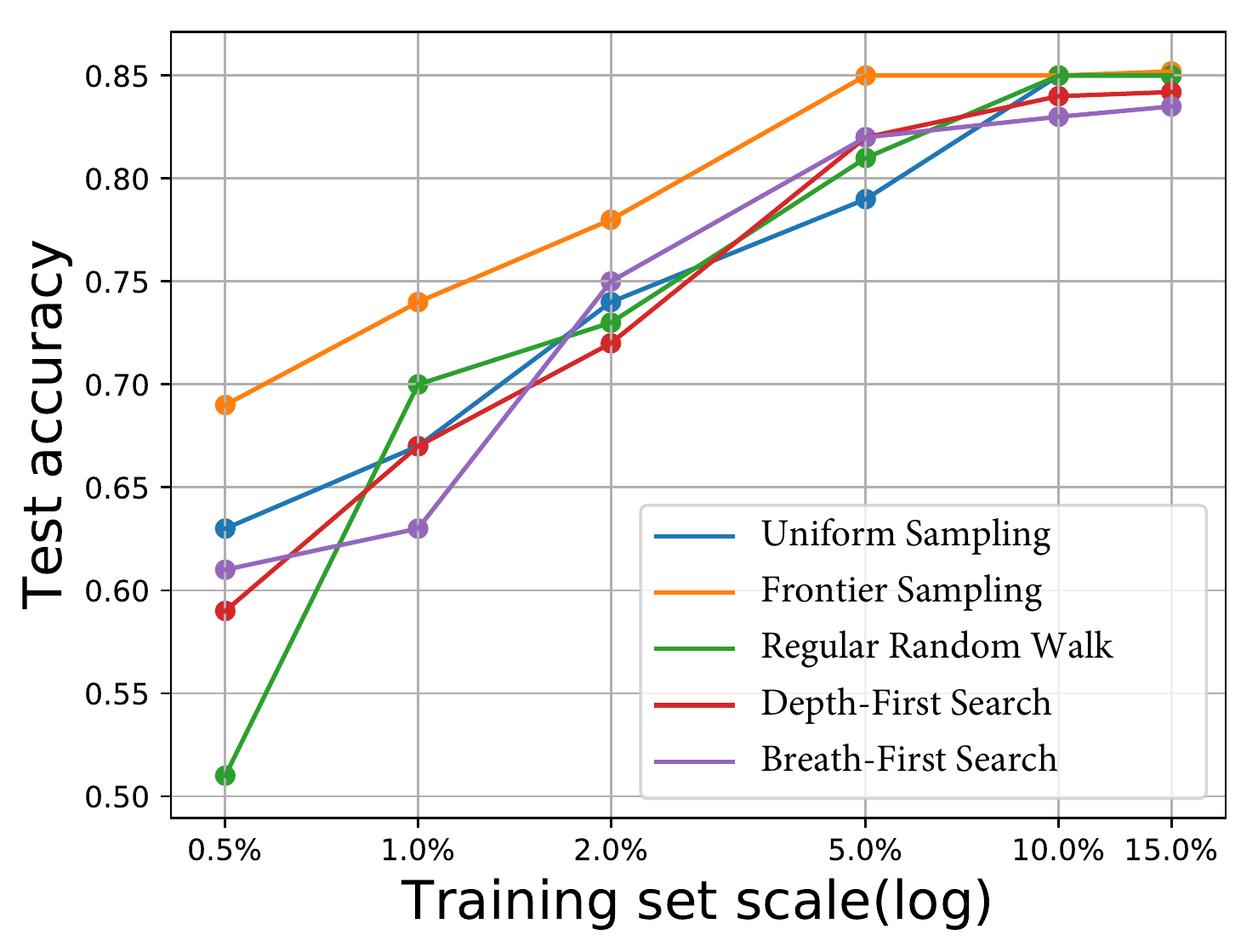}
  \caption{Comparison of different sampling strategy .}
  \label{sampling_comparision}
  \vspace{-0.5cm}
\end{figure}

We do a case study on the Cora dataset with GraphSAGE to verify the efficiency of different sampling strategies, including `Frontier Sampling', `Uniform Sampling', `Regular Random Walks', `Depth-First Search', and `Breadth-First Search'.
The results in Fig \ref{sampling_comparision} indicate that the strategy of `Frontier Sampling' adopted in this paper achieves better results than others.
Regular random walks performs the worst when the sampling scale shrinks to $0.5\%$, caused by its nature of `easily been trapped', but its accuracy rises dramatically when the scale gets larger. 
To be noticed, when the training scale gets larger, uniformly sampled data can get close accuracy with ours, which meets the `Law of Large Number'\cite{hsu1947complete}.

\subsection{Model Compression}\label{sec:5.7}

\begin{figure}[!t]
  \centering
  \subfloat[Additional Time Cost\label{compress_time}]{%
  \includegraphics[width=0.24\textwidth]{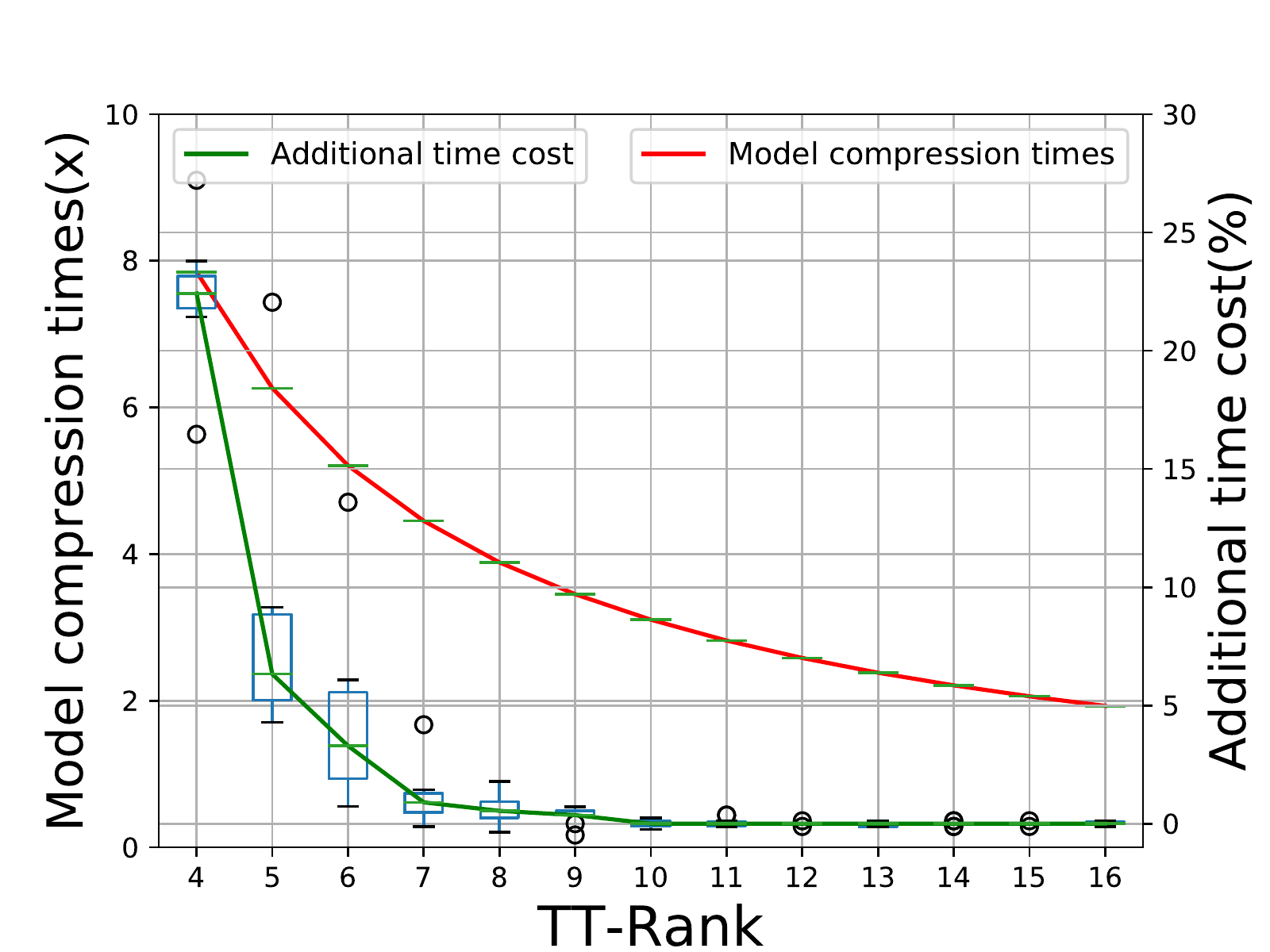}
  }
  \subfloat[Accuracy Decrease \label{compress_accuracy}]{%
  \includegraphics[width=0.24\textwidth]{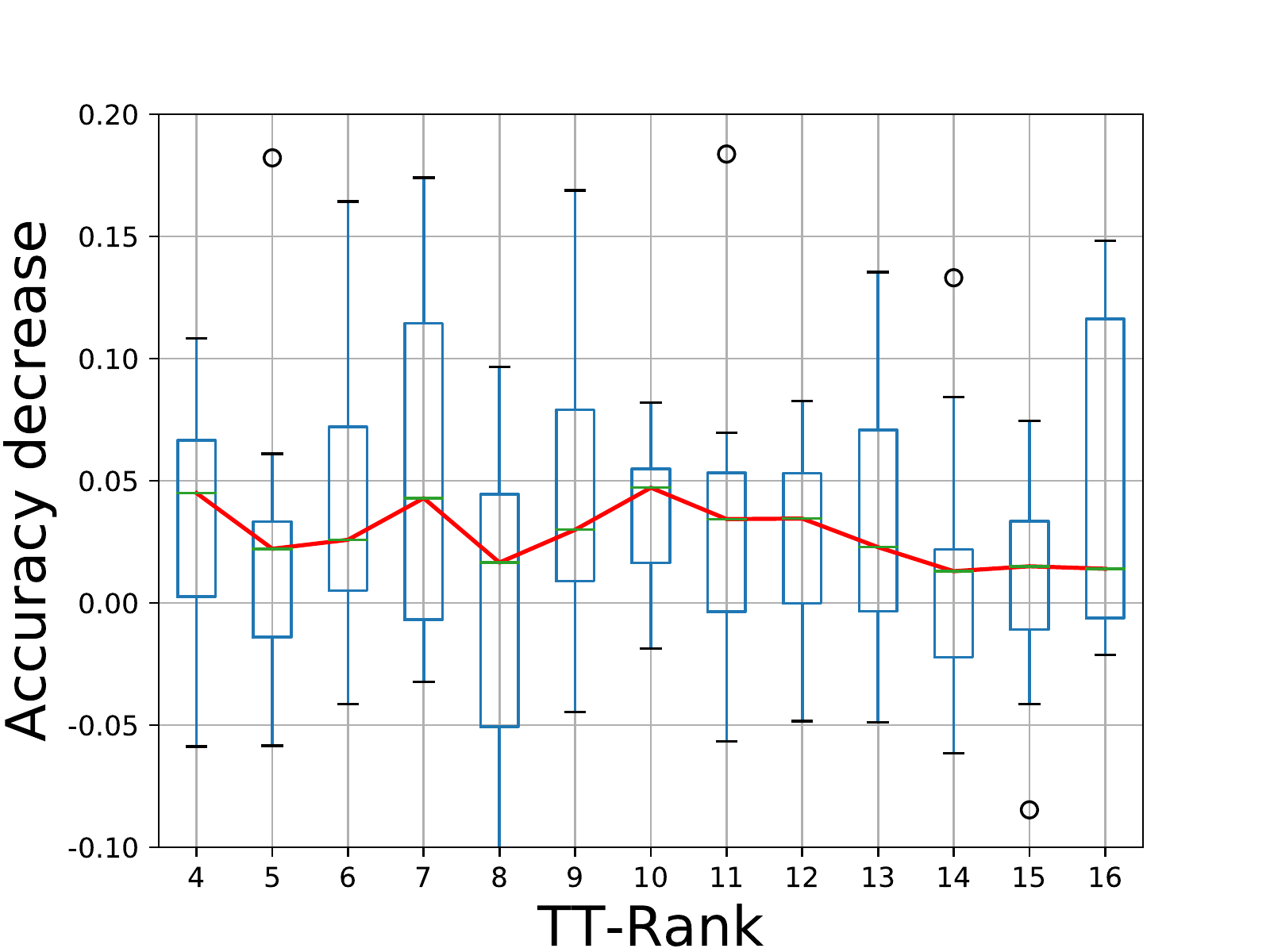}
  }
  \caption{Influence analysis on the additional time cost and prediction accuracy by the model compression.}
  \vspace{-0.5cm}
\end{figure}

To reveal the influence of model compression on the GCN-model, additional time-cost and model accuracy are evaluated.
Fig. \ref{compress_time} shows that compression can reduce the scale of a GCN model with an acceptable additional time-cost increase. 
It costs an additional 16$\%$ computation time to exchange more than 6x parameter reduction of the model.
Fig. \ref{compress_accuracy} shows the model accuracy changes with the TT-rank $r$ rising. 
It can be seen that the average accuracy changes are around 0.02.




\subsection{Algorithm Efficiency} \label{sec:4.4}
The well-distributed training data and fewer parameters to estimate would help the model converging quickly,  which is a common way to reduce the training time.
To evaluate the contribution to reducing the training time, we carried out a case study on GraphSAGE and SS-GraphSAGE with the Pubmed dataset.
We set 10$\%$ of the dataset as the training data, and the number of training epochs is set as 200.

\begin{figure}[!t]
  \centering
  \includegraphics[width=0.25\textwidth]{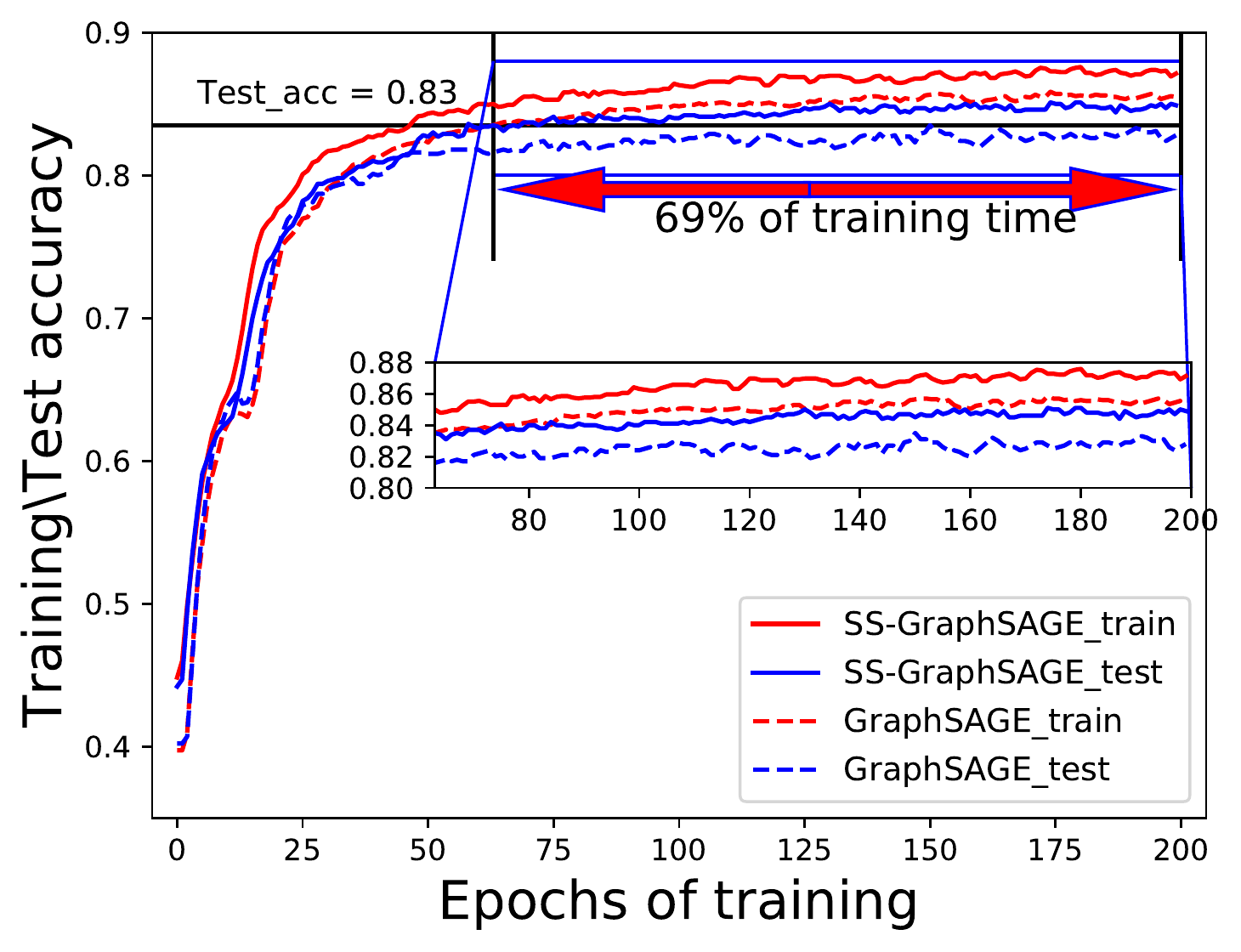}
  \caption{Comparison of convergence speed between the SS-GraphSage with the original one.}
  \label{coverage_speed}
  \vspace{-0.3cm}
\end{figure}

The convergence speeds of  GraphSAGE and SS-GraphSAGE are shown in Fig.\ref{coverage_speed}.
The results indicate that both the training accuracy and test accuracy of SS-GraphSAGE consistently outperform the origin GraphSAGE.
If we set the final test accuracy--0.83 as a threshold, SS-GraphSAGE can achieve a similar test accuracy with only 62 epochs, which decreases the training time by 69$\%$.

\subsection{Parameter Influence}
We take a numerical evaluation on the influence of sampling scale $B$ and the number of random walks dimensions $m$ with SS-GCN on the Cora dataset.

\subsubsection{Sampling Scale}
Fig. \ref{sample_scale} shows the accuracy distance from the steady performance with different sampling budget $B$ on the multi-label classification task.
The steady accuracy is obtained by taking 50$\%$ of the nodes as training data.
From the result, we can observe that the sampling-based methods can easily approach the steady performance with about 1$\%$ to 3$\%$ of nodes as sampling training data.
It reveals the power of the training framework added to the GCN-based methods.

\begin{figure}[!t]
  \centering
  \subfloat[Sampling Scale \label{sample_scale}]{%
  \includegraphics[width=0.23\textwidth]{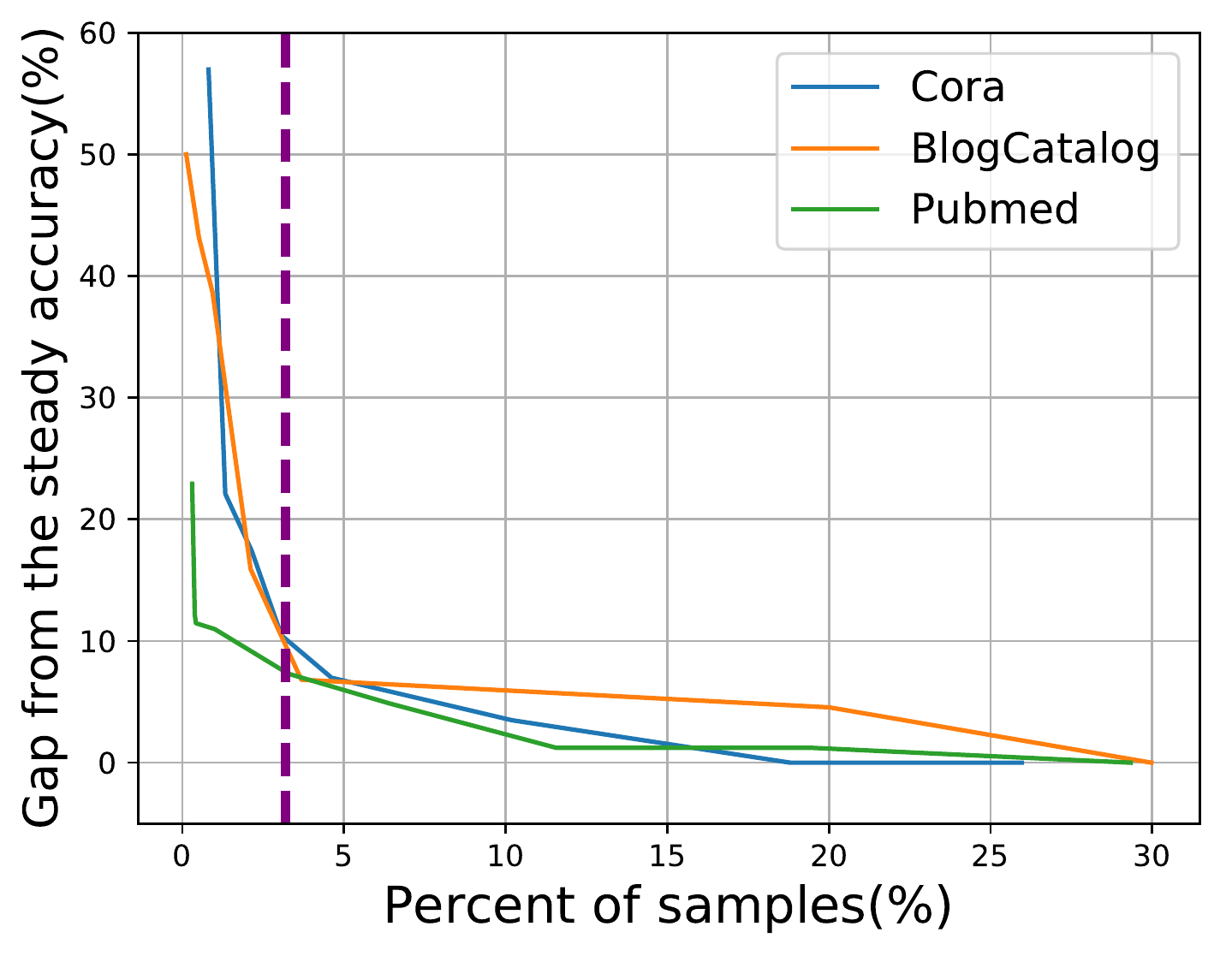}
  }
  \subfloat[Number of Walkers\label{walker_number}]{%
  \includegraphics[width=0.26\textwidth]{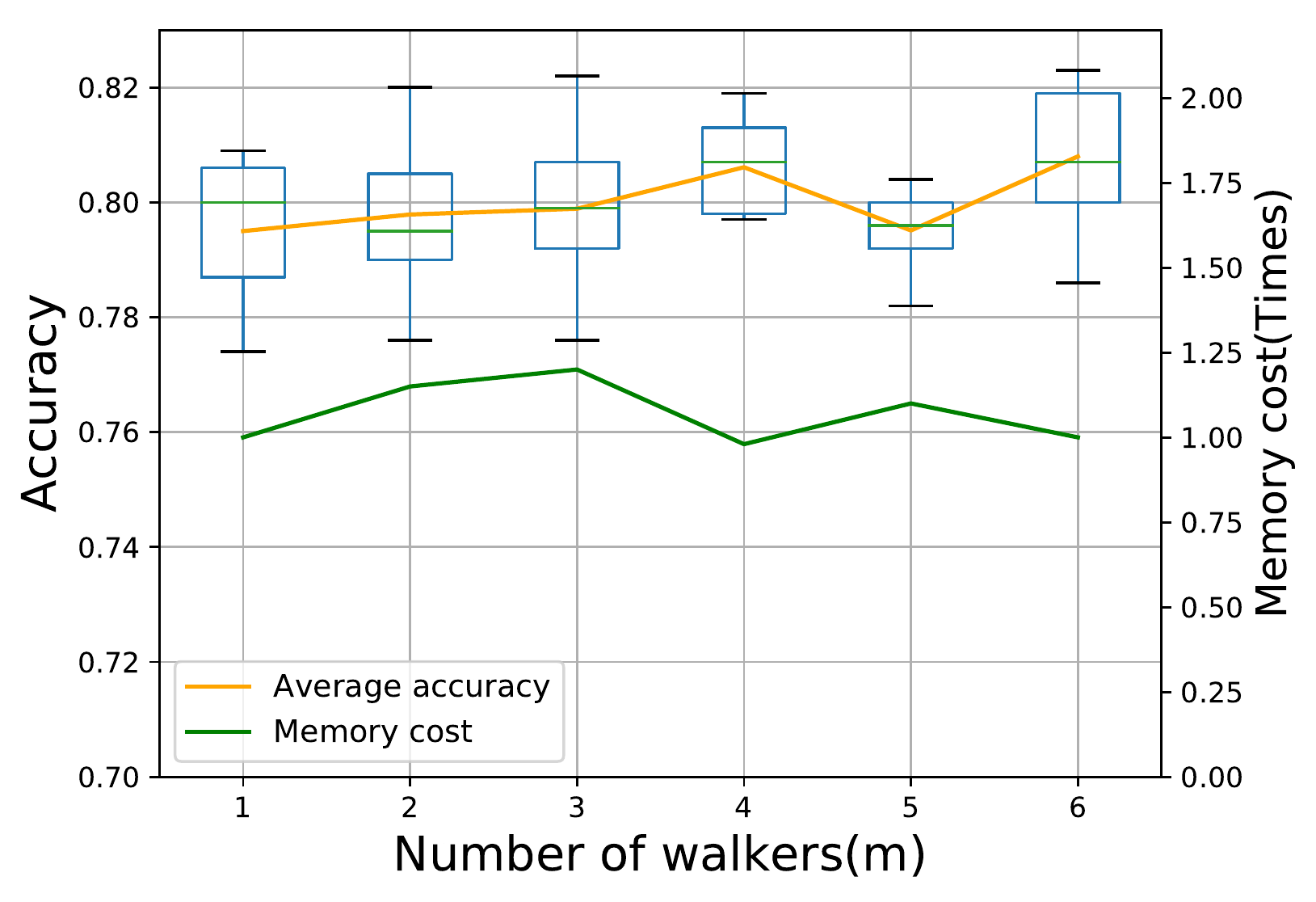}
  }
  \caption{Influence of sampling scale and dimension on the performance of SS-GCN.}
  \vspace{-0.5cm}
\end{figure}

\subsubsection{Number of Walkers}
The sampling scheme performs a $m$ dimensional random walks.  
Fig. \ref{walker_number} shows the influence of the number of walkers $m$ on accuracy and memory cost. 
The box figure shows the distribution of accuracy, the green line is the median accuracy, and the yellow line across the boxes connects the average accuracy. It shows small fluctuations with the changing value of $m$. 
The memory cost is also plotted in the same figure, which stays at the same level as the $m$ ranging.

\section{Related Work}\label{sec:6}
Two lines of research are related to our work, which is summarized as follows.
\subsection{GCN-based Methods} 
Graph neural networks have attracted a lot of attention recently, and various outstanding works are proposed \cite{bruna2014spectral,defferrard2016convolutional,duvenaud2015convolutional}. 
Kipf's  GCN\shortcite{kipf2017semi} has brought it under the spotlight. 
Since then, researchers seek to build a more effective network structure. 
For example, GraphSAGE\cite{hamilton2017inductive} is proposed to deal with the dynamic graphs, and Graph Attention Network\cite{velivckovic2017graph} is proposed to weight the node's neighbors. 
Some works also focus on the problem of training efficiency.
FastGCN\cite{chen2018fastgcn} is one of the pioneers to accelerate the training process by eliminating part of neurons. 
SGC\cite{wu2019simplifying} steps further by simplifying convolutional computation.

\subsection{ Graph Sampling Based on Random Walks}
Sampling methods, especially random walk-based graph sampling methods, have been
widely studied\cite{avrachenkov2010improving,ribeiro2012sampling,xu2014general}. 
Leskovec \textit{et al.} \shortcite{leskovec2006sampling} come up with an efficient way to down-size the sampling scale based on random walks. 
Wei \textit{et al.}\shortcite{wei2004towards} work out how to make the sampling by random walks more efficient. 
Random walks-based sampling can also be used in overlay networks\cite{massoulie2006peer}.
Based on the graph structure's prior knowledge, Zhao \textit{et al.} \cite{zhao2019sampling}  proposed a biased graph sampling strategy by random walks with indirect jumps.

\section{Conclusion} \label{sec:7}
Faced with the challenge of extremely limited annotations for GCN-based methods, we propose an efficient training framework to improve their performance. 
The framework integrating a sampling strategy with a model compression method can obtain well-distributed training data and lower the scale of parameters to estimate, which can significantly improve GCN-based methods' performance.
Six popular GCN baselines are chosen to conduct extensive experiments on three real-world datasets. 
The results indicate that by applying our method, all GCN baselines cut down the annotation requirement by as much as 90$\%$ and compress the scale of parameters more than 6$\times$  without sacrificing their strong performance.


\bibliography{ref}

\end{document}